\documentclass{article} 
\usepackage{iclr2024_conference,times}

\usepackage{hyperref}
\usepackage{url}

\usepackage[utf8]{inputenc} 
\usepackage[T1]{fontenc}    
\usepackage{hyperref}       
\usepackage{url}            
\usepackage{booktabs}       
\usepackage{amsfonts}       
\usepackage{nicefrac}       
\usepackage{microtype}      
\usepackage{xcolor}         
\usepackage{dsfont}
\usepackage{amssymb}
\usepackage{amsthm}

\usepackage[utf8]{inputenc} 
\usepackage[T1]{fontenc}    
\usepackage{booktabs}       
\usepackage{amsfonts}       
\usepackage{nicefrac}       
\usepackage{microtype}      
\usepackage{xcolor}         
\usepackage{graphicx}
\usepackage{soul}
\usepackage{amsmath}

\usepackage{amsmath,amsfonts,bm}

\def\eqref#1{equation~\ref{#1}}

\def\1{\bm{1}}

\DeclareMathAlphabet{\mathsfit}{\encodingdefault}{\sfdefault}{m}{sl}
\SetMathAlphabet{\mathsfit}{bold}{\encodingdefault}{\sfdefault}{bx}{n}

\def\gH{{\mathcal{H}}}

\def\gX{{\mathcal{X}}}

\newtheorem{definition}{Definition}

\newtheorem{theorem}{Theorem}

\newtheorem{lemma}[theorem]{Lemma}
\newtheorem{corollary}[theorem]{Corollary}

\newcommand{\Mat}{\boldsymbol}
\newcommand{\Set}{\mathcal}

\newcommand{\real}{\mathbb{R}}

\newcommand{\Prob}{\mathbb{P}}
\newcommand{\mean}{\mathbb{E}}
\newcommand{\radem}{\mathcal{R}}
\newcommand{\LS}{L_{\Set{S}}}
\newcommand{\LD}{L_{\Set{D}}}
\newcommand{\LSp}{L_{\Set{S^\prime}}}

\DeclareMathOperator{\ERM}{ERM}
\DeclareMathOperator*{\argmax}{arg\,max}
\DeclareMathOperator*{\argmin}{arg\,min}

\PassOptionsToPackage{hyphens}{url}\usepackage{hyperref}
\hypersetup{colorlinks=true}

\title{Generalization Error Analysis for Sparse Mixture-of-Experts: A Preliminary Study}

\author{Jinze Zhao$^{1}$, Peihao Wang$^{1}$,  Zhangyang Wang$^{1}$\\
$^{1}$Department of Electrical and Computer Engineering, UT Austin
}

\iclrfinalcopy 
\begin{document}

\maketitle

\begin{abstract}
Mixture-of-Experts (MoE) represents an ensemble methodology that amalgamates predictions from several specialized sub-models (referred to as experts). This fusion is accomplished through a router mechanism, dynamically assigning weights to each expert's contribution based on the input data. Conventional MoE mechanisms select all available experts, incurring substantial computational costs. In contrast, Sparse Mixture-of-Experts (Sparse MoE) selectively engages only a limited number, or even just one expert, significantly reducing computation overhead while empirically preserving, and sometimes even enhancing, performance. Despite its wide-ranging applications and these advantageous characteristics, MoE's theoretical underpinnings have remained elusive. In this paper, we embark on an exploration of Sparse MoE's generalization error concerning various critical factors. Specifically, we investigate the impact of the number of data samples, the total number of experts, the sparsity in expert selection, the complexity of the routing mechanism, and the complexity of individual experts. Our analysis sheds light on \textit{how \textbf{sparsity} contributes to the MoE's generalization}, offering insights from the perspective of classical learning theory.
\end{abstract}

\section{Introduction}
\label{sec:intro}
Sparse Mixture of Experts (SMoE) ~\citep{shazeer2017outrageously}, has demonstrated remarkable potential in the realm of neural networks, particularly in the context of scaling model size without incurring significant computational overhead. SMoE achieves this by ingeniously partitioning neural networks into smaller, more focused networks known as \textit{expert} networks, which are then selectively and sparsely combined using a data-dependent neural network referred to as the \textit{router}. This approach results in a substantial increase in the total number of parameters while keeping computational costs relatively stable, owing to the fact that only a small subset of experts is chosen for each data point. Furthermore, SMoE is promising to improve model generalization, particularly when dealing with multi-domain data, as it allows a group of experts to learn collaboratively and generalize to unseen domains compositionally. However, despite these compelling advantages, our theoretical understanding of MoE remains elusive. It is often counterintuitive that increasing model size would enhance generalization, as conventional wisdom suggests the opposite effect. 

To address this paradox, this work presents the generalization bound for the Sparse MoE model, by drawing inspiration from classical learning theory concepts ~\citep{natarajan1989learning, vapnik95, shalev-shwartz_ben-david_2022}. We assert that our generalization bound is \textbf{model-agnostic}, offering \textbf{versatility} and \textbf{applicability} across various expert/router-based models when the upper bounds for the complexity metrics of these base models are incorporated. Our main results can be informally stated as:

\begin{itemize}
\item We delivered a \textbf{generalization error bound} for SMoE, which depends on (only) its high-level structure hyperparameters. Our generalization error bound can be applied to generic SMoE model, regardless of the expert or router implementation. 

\item Our generalization bound is in particular \textbf{sparsity-aware}. More specifically, Theorem \ref{thm:main} shows that the generalization error scales with the 'sparsity pattern' by $O \left( \sqrt{k(1+\log(\frac{T}{k}))} \right)$ where $T$ is the number of all experts and $k$ is the number of selected experts, meaning the generalization error bound will grow tighter once we select less experts. The analysis sheds light on \textit{how \textbf{sparsity} contributes to the MoE's generalization even at increasing total model sizes}, compared to conventional MoEs selecting all available experts.   
\end{itemize}

\subsection{Related Work}
\textbf{Mixture of Experts:} The Mixture of Experts (MoE) model, originally introduced by \cite{jacobs1991adaptive} and later extended by \cite{jordan1994hierarchical}, is a framework that integrates multiple specialized sub-models through a router network. This enables the adaptive weighting of their contributions based on the input data distribution, thereby enhancing the overall predictive performance. Recent advancements in the field, such as the works by \cite{shazeer2017outrageously} and \cite{eigen2013learning}, have harnessed this concept within the realm of Deep Learning. These efforts introduced 'routing sparsity', which significantly reduces computational costs and enables the scaling of models to sizes with trillions of parameters. More recently, \cite{fedus2022switch} and \cite{jiang2024mixtral} pushed the boundaries of sparsity in MoE models by aggressively routing input data to sparse experts while paradoxically preserving performance and scaling up model size in language modeling. 


\noindent \textbf{Theoretical Understanding of MoEs: } Recent work ~\citep{chen2022towards} formally studied how the SMoE layer reduces the training error better than using a single expert, and why such mixture model will not collapse into a single model. Importantly, when training an SMoE layer based on the data generated from the “mixture of class" distribution using gradient descent, the authors proved that each expert of the SMoE model will be specialized to a specific portion of the data (at least one cluster), and meanwhile the router can learn the cluster-center features and route the input data to the right experts. More recently, a series of work ~\citep{nguyen2023convergence,nguyen2023demystifying,nguyen2023general,nguyen2023statistical} tried to establish the convergence rates of density estimation and parameter estimation of MoE by defining Voronoi-based losses which describes the interaction between gating function and experts, and explained why Top-1 gating can lead to faster convergence rates of parameter estimation over other gating mechanisms. However, we opt to formally study the \textit{generalization} benefit of Sparsity in MoE modeling, under learning theory perspectives.
 
\section{Preliminaries} \label{sec:setting}

\textbf{Setup:} We mainly focus on \textit{binary classification} for its convenience of analysis.
Consider a distribution: $\Set{D} \sim \Prob(\Mat{x}, y)$, where $\Mat{x} \in \real^d$ is a the raw feature vector, and $y \in \{+1, -1\}$ is the corresponding binary label. 
Same with the classical learning settings, we assume our dataset $\Set{S}$ consists of $m$ i.i.d. samples, i.e., 
\begin{align}
\Set{S} = \{(\Mat{x}_1, y_1), \cdots, (\Mat{x}_m, y_m)\} \overset{i.i.d.}{\sim} \Set{D}^{m}
\end{align}
Given a learner/classifier $f: \real^d \rightarrow \{+1, -1\}$, we define the empirical loss over training set to be
$\Set{S}$: $\LS(f) = \frac{1}{m} \sum_{i=1}^{m} \ell(f(\Mat{x}_i), y_i)$,
and the truth error evaluated over the whole data domain: 
$\LD(f) = \mean_{(\Mat{x}, y) \sim \Set{D}} \ell(f(\Mat{x}), y)$,
where $\ell: \{+1, -1\}^2 \rightarrow \real$ is the loss function.
We define a class of functions as the hypothesis space $\Set{F}$.
At the training stage, our learning algorithm searches for a learner $f \in \Set{F}$ which minimizes the empirical loss, i.e., running Empirical Risk Minimization (ERM): 
\begin{align}
\ERM(\Set{S}) = \argmin_{f \in \Set{F}} \LS(f),
\end{align}
At the testing stage, we compute $\LD(\ERM(\Set{S}))$ to assess the model performance.
The generalization error of hypothesis $f \in \Set{F}$ is: 
\begin{align}
\epsilon(f, \Set{S}) = \left\lvert \LS(f) - \LD(f) \right\rvert
\end{align}

\subsection{Notations for Sparse Mixture-of-Experts}

We define the (simplified) SMoE ~\citep{shazeer2017outrageously} model $f: \real^d \rightarrow \real$ as below:
\begin{align} \label{eqn:smoe}
f(\Mat{x}) = \sum_{j=1}^{T} a(\Mat{x})_j h_j(\Mat{x}) \quad \text{subject to } \sum_{j=1}^{T} \mathds{1}\{a(\Mat{x})_j \ne 0\} = k \quad \forall \Mat{x} \in \real^{d},
\end{align}
where $a(\Mat{x}): \real^d \rightarrow \real^T$ is the router function satisfying $\lVert a(\Mat{x}) \rVert_0 = k$, and $h_j(\Mat{x}): \real^d \rightarrow \real$ is a expert learner for every $j = 1, \cdots, T$.
Intuitively, the entire SMoE model contains $T$ expert learners, thus, its capacity is $T$ times larger as one expert learner.
However, at each forward pass, only $k$ experts will be activated for each data sample, leading to comparably low inference costs as activating all experts.

In this paper, we consider the router function proposed by \cite{shazeer2017outrageously}, which first chooses the $k$ output logits, sets the remainders to zeros, and applies softmax to normalize the chosen entries:
\begin{align} \label{eqn:gate_output}
a(\Mat{x})_j = \left\{
\begin{array}{ll}
\frac{\exp(g(\Mat{x})_j)}{\sum_{t \in \Set{J}(\Mat{x})} \exp(g(\Mat{x})_t) } & \text{if } j \in \Set{J}(\Mat{x}) \\ \\
0 & \text{if } j \notin \Set{J}(\Mat{x})
\end{array}\right. ,
\end{align}
where $g: \real^d \rightarrow \real^T$ computes the weights for each expert, and $\Set{J}(\Mat{x})$ finds a sparse mask with at most $k$ non-zero entries according to $\Mat{x}$, i.e., $\lvert \Set{J}(\Mat{x}) \rvert = k, \forall \Mat{x} \in \real^{d}$.
For example, $\Set{J}(\Mat{x})$ selects the indices of top-$k$ largest logits from $g(\Mat{x})$.

Below, we formally state the hypothesis space of an SMoE, which is composed of the hypotheses of both router and expert learners.
\begin{definition} \label{dfn:smoe}
Suppose all expert learners $h_1, \cdots, h_T \in \Set{H}$ is selected from the same hypothesis space $\Set{H}$, and $k$-sparse router function $a \in \Set{A}$ is chosen from the hypothesis space $\Set{A}$. Define the hypothesis space of the SMoE model with $T$ experts and $k$-sparse router function as below:
\begin{align}
\Set{F}(T, k) = \left\{ f(\Mat{x}) = \sum_{j=1}^{T} a(\Mat{x})_j h_j(\Mat{x}) : h_1, \cdots, h_T \in \Set{H}, a \in \Set{A} \right\}
\end{align}
\end{definition}

\subsection{Complexity Metrics and Main Proof Tools}

Given a space $Z$ and a fixed distribution $D$, let $S=\left\{z_1, \ldots, z_m\right\}$ be a set of examples drawn i.i.d. from $D$. Furthermore, let $\mathcal{F}$ be a class of functions $f: Z \rightarrow \mathbb{R}$.
\begin{definition}[Rademacher complexity] \label{dfn:radem_complexity}
The Rademacher complexity of $\Set{F}$ is defined as $\radem_m(\mathcal{F}) = \mean_{\Set{S} \sim \Set{D}^m} \left[ \radem_m(\Set{F}, \Set{S})   \right]$, where $\radem_m(\Set{F}, \Set{S})=\mean_\sigma\left[\sup _{f \in \mathcal{F}}\left(\frac{1}{m} \sum_{i=1}^m \sigma_i f\left(z_i\right)\right)\right]$ is the Empirical Rademacher complexity with $\sigma_1, \ldots, \sigma_m$ to be Rademacher random variables.

\end{definition}


\begin{definition}[Natarajan Dimension] \label{dfn:natarajan_dim}
The Natarajan Dimension of $\gH$, denoted by $d_N(\Set{H})$, is the maximal size of a multiclass-shattered set $C \subset \gX$. We say that a set $C \subset \mathcal{X}$ is shattered ~\citep{shalev-shwartz_ben-david_2022} by $\mathcal{H}$ if there exist two functions $f_0, f_1: C \rightarrow[k]$ such that
    For every $x \in C, f_0(x) \neq f_1(x)$.
    For every $B \subset C$, there exists a function $h \in \mathcal{H}$ such that
    $\forall x \in B, h(x)=f_0(x) \text { and } \forall x \in C \backslash B, h(x)=f_1(x)$.
\end{definition}
Moreover, we define an Indicator function that characterizes the sparse patterns produced by the router function $a \in \Set{A}$.
\begin{definition}\label{dfn:natarajan_dim_router}
Define the following function:
\begin{align}
m(\Mat{x})_j = \left\{
\begin{array}{ll}
1 & \text{if } a(\Mat{x})_j \ne 0 \\ \\
0 & \text{if } a(\Mat{x})_j = 0
\end{array}\right.
\end{align}
We use constant $d_N$ to represent the the maximal size of a set that shatters all possible outcomes of $m(\Mat{x})$.
\end{definition}

\section{Main Results}

We follow the notations that we defined above and state our generalization error bound as below:
\begin{theorem} \label{thm:main}
Suppose the loss function $\ell: \Set{Y} \times \real \rightarrow [0, 1]$ is $C$-Lipschitz, and the hypothesis space $\Set{F}(T, k)$ follows Definition \ref{dfn:smoe}, then with probability at least $1 - \delta$ over the selection of training samples, the generalization error is upper bounded by:
\begin{align}
O\left(4C\radem_m(\Set{H}) + 2\sqrt{\frac{2 k d_N (1+\log (\frac{T}{k})) + d_N \log(2m) + \log(4 / \delta)}{2m}}\right),
\end{align}
where $\radem_m(\Set{H})$ is the Rademacher complexity of the expert hypothesis space $\Set{H}$ (cf. Definition \ref{dfn:radem_complexity}), $d_N$ is the Natarajan dimension of router function hypothesis space $\Set{A}$ (cf. Definition \ref{dfn:natarajan_dim_router}), $m$ is the number of training samples, $T$ is the total number of experts, and $k$ is the number of selected experts.
\end{theorem}
We defer the proof to Appendix \ref{sec:proof}. Below we present the applications to neural networks. 

\subsection{Application to Neural Networks}

We highlight that our generalization bound is model-agnostic and can be applied to SMoE with different types of router or expert models, once we approximate or upper bound their corresponding complexity metrics. We now instantiate it for neural networks.


\begin{lemma}[Natarajan Dimension Bound of NN with ReLU activations ~\citep{jin2023upper}] \label{lem:nat_nn}
    Consider $\Pi_{p, S}^{\text {ReLU}}$, a neural network function class with a fixed structure $S$ of $p$ parameters (structure $S$ has $L$ layers, where the $\ell$-th layer has $n_{\ell}$ nodes, $\ell \in\{1, \ldots, L\}$), then $d_N\left( \Pi_{p, S}^{\text {ReLU}}\right) \leq O(d \cdot p^2)$, where $d$ is the number of outputs in the final layer.
\end{lemma}
The next Rademacher Complexity Bound is useful for analyzing multilayer NNs ~\citep{bartlett2017spectrally}:
\begin{lemma}[Rademacher Complexity Bound for NN ~\citep{bartlett2017spectrally}] \label{lem:rade_nn}
Assume $\left\|\Mat{x}^{(i)}\right\|_2 \leqslant c, \quad \forall i$, and let $\Set{H}=\left\{h_{\vartheta}:\left\|\Mat{W}_i\right\|_{op} \leq K_i,\left\|\Mat{W}_i^{\top}\right\|_{2,1} \leqslant b_i\right\}$, where $h_{\vartheta}(\Mat{x})=\Mat{W}_r \phi\left(\Mat{W}_{r-1} \phi\left(\Mat{W}_{r-2} \ldots \phi\left(\Mat{W}_1 x\right) \ldots\right)\right.$, $\left\|\Mat{W}_i\right\|_{op}$ is the spectral norm of $\Mat{W}_i$, $\left\|\Mat{W}_i^{\top}\right\|_{2,1}$ is the sum of the $l_2$ norms of the rows of $\Mat{W}_i$, and $r$ is the number of layers. Then we have $\radem_m(\Set{H}) \leq \frac{c}{\sqrt{m}} \cdot\left(\prod_{i=1}^r K_i\right) \cdot\left(\sum_{i=1}^r \frac{b_i^{2 / 3}}{K_i^{2 / 3}}\right)^{3 / 2}$.
\end{lemma}

By plugging Lemma \ref{lem:rade_nn} and Lemma \ref{lem:nat_nn} in our Theorem \ref{thm:main}, we conclude the following NN-specific statement:
\begin{corollary} \label{cor:cor1}
Suppose the loss function $\ell: \Set{Y} \times \real \rightarrow [0, 1]$ is $C$-Lipschitz, and each expert is a neural network satisfying the definition in Lemma \ref{lem:rade_nn}, and router network is a neural network satisfying the definition in Lemma \ref{lem:nat_nn}, then with probability at least $1 - \delta$ over the selection of training samples, the generalization error is upper bounded by:
\begin{align}
O\left(4C \frac{c}{\sqrt{m}} \cdot\left(\prod_{i=1}^r K_i\right) \cdot\left(\sum_{i=1}^r \frac{b_i^{2 / 3}}{K_i^{2 / 3}}\right)^{3 / 2} + 2\sqrt{\frac{2k dp^2 (1 + \log(\frac{T}{k})) + dp^2 \log(2m) + \log(4 / \delta)}{2m}}\right)
\end{align}
\end{corollary}


\paragraph{Remark on Sparsity Awareness:} 
We claimed that our generalization bound is \textit{sparsity-aware}. Inspecting the generalization bound derived in Theorem \ref{thm:main} and Corollary \ref{cor:cor1}, we remark that the term $O \left( \sqrt{k(1+\log(\frac{T}{k}))}  \right)$ is strategic and somehow explains how Sparse MoE generalizes. We notice that the term is monotonically increasing with $k$ in the interval $k \le T$. While the model grows larger with more available experts in parallel, it will harm the generalization by $O \left( \log (T) \right)$ factor, but we can still compensate the generalization but simply choosing less experts to counteract the problem because the error term will scale down by a factor $O (k)$.

\section{conclusion}
We presented the generalization error bound for the Sparse Mixture-of-Experts model. Our bound is model-agnostic and can be applied to any SMoE structure with different router/expert base models, once we can derive their complexity metrics. More importantly, our bound explains why choosing less experts can help SMoE generalize better. Nevertheless, we acknowledge that the generalization bound can be further tightened if we introduce more complex features of SMoE such as dynamic routing. In future work, we plan to combine these features with latest deep learning theory tools to further study this intriguing SMoE generalization problem.

\bibliography{iclr2024_conference}
\bibliographystyle{iclr2024_conference}

\appendix
\section{Appendix} \label{sec:proof}
\subsection{Proof of Theorem \ref{thm:main}}
The main objective is to show the following probabilistic bound of $\sup_{f \in \Set{F}} \left\lvert \LS(f) - \LD(f) \right\rvert$.
We first draw an i.i.d. copies of training samples: $\Set{S'} = \{\Mat{x'}_1, \cdots, \Mat{x'}_m\} \overset{i.i.d.}{\sim} \Set{D}^m$.
Then by Lemma \ref{lem:ghost_sample}, and setting $e^{-\frac{1}{2} \epsilon^2 m} \le 1/4$, we have

\begin{equation}
\Prob \left(\sup_{f \in \Set{F}} \left\lvert \LS(f) - \LD(f) \right\rvert \ge \epsilon \right) \le 2 \Prob \left(\sup_{f \in \Set{F}} \left\lvert \LS(f) - \LSp(f) \right\rvert \ge \frac{\epsilon}{2} \right) \label{eqn:starter_ineq}
\end{equation}

Our proof proceeds by reformulating the router function $a(\Mat{x})$.
Let us rewrite $a(\Mat{x}) = \mu(\Mat{x}) \odot \nu(\Mat{x})$, where $\odot$ denotes the element-wise multiplication, $\mu(\Mat{x}): \real^{d} \rightarrow \{0, 1\}^T$ produces a binary mask specifying the sparse expert selection ($\lVert \mu(\Mat{x}) \rVert_0 = k$), and $\nu(\Mat{x}): \real^{d} \rightarrow \real_+^T$ outputs the normalized weights for selected experts such that $\lVert a(\Mat{x}) \rVert_1 = 1$. Noting that $\nu(\Mat{x})$ is dependent of the function $\mu(\Mat{x})$, we define $\Set{V}\lvert_{\mu}$ as the class of $g$ induced by $\Set{A}$ and $\mu(\Mat{x})$.

Now notice that $\mu(\Mat{x})$ amounts to a multi-class classifier, which maps the input $\Mat{x}$ to one of the sparse patterns $\Set{M}(2m)$.
Define $\mu_1, \cdots, \mu_{\Gamma}$ which shatters all the possible sparse patterns produced by $2m$ data samples, then we have: 
\begin{align}
&\Prob \left(\sup_{f \in \Set{F}} \left\lvert \LS(f) - \LSp(f) \right\rvert \ge \frac{\epsilon}{2} \right)\\
&= \Prob \left(\sup_{a \in \Set{A}} \sup_{\substack{h_j \in \Set{H}, \\ \forall j \in [T]}} \left\lvert \LS(f) - \LSp(f) \right\rvert \ge \frac{\epsilon}{2} \right) \label{eqn:F_space}\\ 
&\le \Prob \left(\sup_{\mu \in \{\mu_1, \cdots, \mu_{\Gamma}\}} \sup_{\substack{h_j \in \Set{H}, \forall j \in [T] \\ \nu \in \Set{V}\lvert_{\mu}}} \left\lvert \LS(f) - \LSp(f) \right\rvert \ge \frac{\epsilon}{2} \right) \label{eqn:discrete_mask_space}\\
&\le \sum_{t = 1}^{\Gamma} \Prob \left( \left. \sup_{\substack{h_j \in \Set{H}, \forall j \in [T] \\ \nu \in \Set{V}\lvert_{\mu}}} \left\lvert \LS(f) - \LSp(f) \right\rvert \ge \frac{\epsilon}{2} \right| \mu = \mu_t \right) \\
&\le 2 \Gamma \sup_{\mu^* \in \{\mu_1, \cdots, \mu_{\Gamma}\}} \Prob \left( \left. \sup_{\substack{h_j \in \Set{H}, \forall j \in [T] \\ v \in \Set{V}\lvert_{\mu}}} \LS(f) - \LSp(f) \ge \frac{\epsilon}{2} \right| \mu = \mu^* \right) \label{eqn:remove_abs}
\end{align}

where Eq. \ref{eqn:F_space} is obtained by Definition \ref{dfn:smoe}, Eq. \ref{eqn:discrete_mask_space} and Eq. \ref{eqn:remove_abs} are obtained by union bound.
Next, we do counting to bound $\Gamma$. Since $\mu$ is essentially a multi-class classifier with $T \choose k$ number of possible patterns, we will consider routing as a multi-class prediction problem by choosing one specific pattern among all possible patterns. By Lemma \ref{lem:nataranjan}, we plug in the Natarajan dimension of our sparse patterns:
\begin{equation} \label{eqn:growth_bound}
\Gamma \le {T \choose k}^{2d_N} \cdot (2m)^{d_N}.
\end{equation}
On the other hand, we bound remaining probability term by examining the expectation given a fixed masking function $\mu$. Define function $\phi\lvert_{\mu}$ conditioned on $\mu$ as:
\begin{equation} \label{eqn:phi_def}
\phi\lvert_{\mu}(\Set{S}, \Set{S'}) = \sup_{\substack{h_j \in \Set{H}, \forall j \in [T] \\ \nu \in \Set{V}\lvert_{\mu}}} \LS(f) - \LSp(f) 
\end{equation}
By McDiarmid's inequality, for any $\mu$, we have bound:
\begin{align}
&\Prob \left( \phi\lvert_{\mu}(\Set{S}, \Set{S'}) \ge \frac{\epsilon}{2} \right) \\ &= \Prob \left( \phi\lvert_{\mu}(\Set{S}, \Set{S'}) - \mean\left[ \phi\lvert_{\mu}(\Set{S}, \Set{S'}) \right] \ge \frac{\epsilon}{2} - \mean\left[ \phi\lvert_{\mu}(\Set{S}, \Set{S'}) \right] \right) \\
&\le \exp\left( -2m\left(\frac{\epsilon}{2} - \mean\left[ \phi\lvert_{\mu}(\Set{S}, \Set{S'}) \right]\right)^2 \right) \label{eqn:mcdiarmid_bound}
\end{align}
Combined with Eq. \ref{eqn:starter_ineq} \ref{eqn:remove_abs}, \ref{eqn:growth_bound}, \ref{eqn:mcdiarmid_bound}, we state, with probability at least $1 - \delta$,
\begin{align}
&\Prob \left(\sup_{f \in \Set{F}} \left\lvert \LS(f) - \LD(f) \right\rvert \ge \epsilon \right) \\ &\le 2 \mean\left[ \phi\lvert_{\mu}(\Set{S}, \Set{S'}) \right] + 2\sqrt{\frac{\log\left({T \choose k}^{2d_N} (2m)^{d_N} \right) + \log(4 / \delta)}{2m}} \\
&\le 2\mean\left[ \phi\lvert_{\mu}(\Set{S}, \Set{S'}) \right] + 2\sqrt{\frac{2 k d_N (1+\log(\frac{T}{k})) + d_N \log(2m) + \log(4 / \delta)}{2m}} \label{eqn:T_choose_k_bound}
\end{align}
where we used the inequality ${T\choose k} \leq \left(\frac{e T}{k}\right)^k$ in Eq. \ref{eqn:T_choose_k_bound}. We conclude the proof by the bounding $\mean\left[ \phi\lvert_{\mu}(\Set{S}, \Set{S'}) \right] \le 2C \radem_m(\Set{H})$ using Lemma \ref{lem:bound_mean}.

\begin{lemma} \label{lem:bound_mean}
Consider C-Lipschitz loss function:: $\ell: \Set{Y} \times \real \rightarrow \real$, and follow the definition of $\phi\lvert_{\mu}$ in Eq. \ref{eqn:phi_def}, we have
\begin{align}
\mean\left[ \phi\lvert_{\mu}(\Set{S}, \Set{S'}) \right] \le 2C \radem_m(\Set{H})
\end{align}
\end{lemma}
\begin{proof}
For the sake of notation simplicity, we define a function space conditioned on a masking function $\mu$:
\begin{align}
\Set{F}\lvert_{\mu} = \left\{ f(\Mat{x}) = \sum_{j=1}^{T} \mu(\Mat{x})_j \nu(\Mat{x})_j h_j(\Mat{x}) : h_1, \cdots, h_T \in \Set{H}, \nu \in \Set{V}\lvert_{\mu} \right\}
\end{align}
We denote the loss function $\ell$ composed on $\Set{F}\lvert_{\mu}$ as $\ell \circ \Set{F}\lvert_{\mu} = \left\{ \ell(f(\Mat{x})) : f \in \Set{F}\lvert_{\mu} \right\}$.
By Lemma \ref{lem:radem_bound},
\begin{align}
\mean \left[ \phi\lvert_{u}(\Set{S}, \Set{S'}) \right] &= \mean \left[ \sup_{\ell_f \in \ell \circ \Set{F}\lvert_{\mu}} \left( \frac{1}{m} \sum_{i=1}^{m} \ell_f(\Mat{x}_i) - \frac{1}{m} \sum_{i=1}^{m} \ell_f(\Mat{x'}_i) \right) \right] \\
&\le 2\radem_m(\ell \circ \Set{F}\lvert_{\mu}) \label{eqn:E_radem_bound}
\end{align}
Since $\ell$ is Lipschitz function, by Lemma \ref{lem:lips_radem}, we have
\begin{align} \label{eqn:lips_radem}
\radem_m(\ell \circ \Set{F}\lvert_{\mu}) \le C \radem_m(\Set{F}\lvert_{\mu}) 
\end{align}
Afterwards, we bound $\radem_m(\Set{F}\lvert_{\mu})$ by:
\begin{align}
&\mean_{\Set{S}, \Mat{\sigma}} \left[ \frac{1}{m} \sup_{f \in \Set{F}\lvert_{\mu}} \sum_{i=1}^{m} \sigma_i f(\Mat{x}_i) \right] \\ &= \mean_{\Set{S}, \Mat{\sigma}} \left[ \sup_{\sup_{\substack{h_j \in \Set{H}, \forall j \in [T] \\ v \in \Set{V}\lvert_{\mu}}}} \frac{1}{m} \sum_{i=1}^{m} \sigma_i \sum_{j=1}^{T} \mu(\Mat{x}_i)_j \nu(\Mat{x}_i)_j h_j(\Mat{x}_i) \right] \\
&\le \mean_{\Set{S}, \Mat{\sigma}} \left[ \sup_{\substack{h_j \in \Set{H}, \forall j \in [T] \\ \Mat{\lambda} \in \real_+^T, \lVert \Mat{\lambda} \rVert_1 = 1}} \frac{1}{m} \sum_{i=1}^{m} \sigma_i \sum_{j=1}^{T} \Mat{\lambda}_j h_j(\Mat{x}_i) \right] \label{eqn:relax_simplex}\\
&= \radem_m(\Set{H}) \label{eqn:conv_radem},
\end{align}
where we notice that $\sum_{j=1}^{T} \mu(\Mat{x})_j \nu(\Mat{x}_j) = 1$ due to the softmax normalization over the weights of selected experts, then Eq. \ref{eqn:relax_simplex} can be relaxed by supremum over all simplex. The last equation follows from Lemma \ref{lem:convex_hull_radem}.
Now we can conclude the proof by combining Eq. \ref{eqn:E_radem_bound}, \ref{eqn:lips_radem}, and \ref{eqn:conv_radem}.
\end{proof}

\begin{lemma}[Natarajan Lemma] \label{lem:nataranjan}
Given a set of finite data points $\Set{S}$ with $\lvert \Set{S} \rvert = m$, and a hypothesis space $\Set{H}$ of functions $\Set{S} \rightarrow [k]$ with Natarajan dimension $d_N$, then the growth function is bounded by:
\begin{align}
\tau_{\Set{H}}(m) \le m^{d_N} \cdot k^{2d_N}
\end{align}
\end{lemma}
\begin{proof}
See \cite{natarajan1989learning}.
\end{proof}

\begin{lemma}[Ghost Sampling] \label{lem:ghost_sample}
Given $\Set{S}$ and $\Set{S^\prime}$ with $\left | \Set{S} \right |= \left | \Set{S^\prime} \right | = m$, we have the following inequality for any hypothesis space $\Set{H}$:
\begin{align}
&\left(1-2 e^{-\frac{1}{2} \epsilon^2 m}\right) \Prob\left[\sup _{h \in \Set{H}}\left|\LS(h)-\LSp(h)\right|>\epsilon\right] \\ &\leq \Prob\left[\sup _{h \in \mathcal{H}}\left|\LS(h)-\LSp(h)\right|>\frac{\epsilon}{2}\right]
\end{align}
\end{lemma}

\begin{proof}
See \cite{shalev-shwartz_ben-david_2022}.
\end{proof}

\begin{lemma} \label{lem:radem_bound}
Given any function class $\Set{F}$, for any $\Set{S}$ and $\Set{S'}$ drawn i.i.d. from $\Set{D}^m$ with $|\Set{S}| = |\Set{S'}| = m$, it holds that
\begin{align}
\mean_{\Set{S}, \Set{S'}} \left[ \sup_{f \in \Set{F}} \left( \frac{1}{m} \sum_{i=1}^{m} f(\Mat{x}_i) - \frac{1}{m} \sum_{i=1}^{m} f(\Mat{x'}_i) \right) \right] \le 2 \radem_m(\Set{F})
\end{align}
\end{lemma}

\begin{proof}
The proof is concluded by the following derivation:
\begin{align}
    &\mean \left[ \sup_{f \in \Set{F}} \LS(f) - \LSp(f) \right] \\
    &= \mean_{\Set{S}} \left[ \mean_{\Set{S^\prime}} \left[\sup_{f \in \Set{F}} \left\lvert \LS(f) - \LSp(f) \right\rvert  \right]   \right]\\
    &\leq \mean_{\Set{S},\Set{S^\prime}} \left[ \mean_{\sigma_{i}} \left[ \sup_{f \in \Set{F}} \left(\frac{1}{m} \sum_{i=1}^m \left(f(z_i) - \frac{1}{m}\sum_{i=1}^m f(z_i^\prime) \right) \right) \right]  \right] \label{eqn:intro_radem_vars} \\
    &\leq \mean_{\Set{S}, \Set{S^\prime}, \sigma_i} \left[ \sup _{f \in \Set{F}}\left(\frac{1}{m} \sum_{i=1}^m \sigma_i f\left(z_i\right)\right)+\sup _{f \in \Set{F}} \left(\frac{1}{m} \sum_{i=1}^m-\sigma_i f\left(z_i^{\prime}\right)\right) 
    \right]\\
    &= 2 \radem_m(\Set{F}),
\end{align}
where we introduce Rademacher random variables $\sigma_i, i = 1, \cdots, m$ in Eq. \ref{eqn:intro_radem_vars}.
\end{proof}

\begin{lemma} \label{lem:lips_radem}
Suppose $\Set{H} \subseteq \{h: \Set{X} \rightarrow \Set{Y} \}$ and function $\ell: \Set{Y} \times \real \rightarrow \real$ is a C-Lipschitz function, define define $\ell \circ \Set{H} = \{\ell \circ h: \forall h \in \Set{H}\}$, then $\radem_m(\ell \circ \Set{H}) \le C\radem_m(\Set{H})$.
\end{lemma}

\begin{proof}
    See \cite{meir_zhang_2003}.
\end{proof}

\begin{lemma} \label{lem:convex_hull_radem}
Suppose $\Set{H} \subseteq \{h: \Set{X} \rightarrow \Set{Y} \}$, then all functions constructed by convex combinations of $\Set{H}$ satisfies:
\begin{align}
\mean_{\Set{S}, \Mat{\sigma}} \left[ \sup_{\substack{h_j \in \Set{H}, \forall j \in [T] \\ \Mat{\lambda} \in \real_+^T, \lVert \Mat{\lambda} \rVert_1 = 1}} \frac{1}{m} \sum_{i=1}^{m} \sigma_i \sum_{j=1}^{T} \Mat{\lambda}_j h_j(\Mat{x}_i) \right] = \radem_m(\Set{H}).
\end{align}
\end{lemma}

\begin{proof}
\begin{align}
&\mean_{\Set{S}, \Mat{\sigma}} \left[ \sup_{\substack{h_j \in \Set{H}, \forall j \in [T] \\ \Mat{\lambda} \in \real_+^T, \lVert \Mat{\lambda} \rVert_1 = 1}} \frac{1}{m} \sum_{i=1}^{m} \sigma_i \sum_{j=1}^{T} \Mat{\lambda}_j h_j(\Mat{x}_i) \right] \\ &= \mean_{\Set{S}, \Mat{\sigma}} \left[ \sup_{\substack{h_j \in \Set{H}, \\ \forall j \in [T]}} \sup_{\substack{\Mat{\lambda} \in \real_+^T, \\ \lVert \Mat{\lambda} \rVert_1 = 1}} \frac{1}{m} \sum_{j=1}^{T} \Mat{\lambda}_j \left(\sum_{i=1}^{m} \sigma_i h_j(\Mat{x}_i)\right) \right] \\
&= \mean_{\Set{S}, \Mat{\sigma}} \left[ \sup_{h_{j^*} \in \Set{H}} \frac{1}{m} \sum_{i=1}^{m} \sigma_i h_{j^*}(\Mat{x}_i) \right] \label{eqn:conv_selec} \\
&= \radem_m(\Set{H}),
\end{align}
where Eq. \ref{eqn:conv_selec} uses the fact that $\sum_{j=1}^{T} \Mat{\lambda}_j y_j \le \max_{j=1,\cdots,T} y_j$ for coefficients $\Mat{\lambda}$. Moreover, the equality is achieved if and only if $\Mat{\lambda}_{j} = 1$ for $j = \argmax_{j=1,\cdots,T} y_j$ and $\Mat{\lambda}_{j} = 0$ otherwise.
\end{proof}

    


\end{document}